\documentclass[twoside]{article}
\usepackage[accepted]{aistats_arxiv}

\usepackage[round]{natbib}

\usepackage{amsmath}
\usepackage{amssymb}
\usepackage{tikz}
\usepackage{tabularx}
\usepackage{amsfonts}

\usepackage{array}
\usepackage{pgfplots}

\usepackage{hyperref}

\usepackage{subfigure}
\usepackage{array}
\usepackage{amsfonts}
\usepackage{microtype}
\usepackage{url}
\usepackage{pdfpages}

\newcommand{\zeroone}{{01}}
\newcommand{\ig}{\mathrm{ig}}

\renewcommand{\epsilon}{\varepsilon}

\newcommand{\ermres}{\hat{h}_{\mathrm{ERM}}}
\newcommand{\wsqrt}{R_{\mathrm{s}}}
\newcommand{\nsqrt}{R_{\mathrm{f}}}
\newcommand{\ermbound}{\mathrm{B}_{\mathrm{ERM}}}
\newcommand{\ourbound}{\mathrm{B}_{\mathrm{PR}}}
\newcommand{\VC}{\mathrm{VC}}

\usepackage{bm}

\usepackage{tikz}
\usepackage{tikz-qtree}

\usetikzlibrary{decorations.pathreplacing}

\usepackage{amsmath,amsfonts,stmaryrd,amsthm,amssymb,enumitem}
\usepackage{algorithm,booktabs}

\usepackage{algorithmicx}
\usepackage[noend]{algpseudocode}
\algnewcommand\algorithmicinput{\textbf{Input:}}
\algnewcommand\INPUT{\item[\algorithmicinput]}

\algnewcommand\algorithmicoutput{\textbf{Output:}}
\algnewcommand\OUTPUT{\item[\algorithmicoutput]}

\usepackage{pifont}

\usepackage{tikz}
\usetikzlibrary{arrows}
\usetikzlibrary{calc}
\usetikzlibrary{shapes}
\usetikzlibrary{fit}
\usetikzlibrary{matrix} 
\usetikzlibrary{positioning}
\usetikzlibrary{decorations.pathreplacing}

\newtheorem{theorem}{Theorem}[section]

\newtheorem{lemma}[theorem]{Lemma}

\renewcommand{\eqref}[1]{Eq.~(\ref{eq:#1})}

\newcommand{\thmref}[1]{Theorem~\ref{thm:#1}}
\newcommand{\lemref}[1]{Lemma~\ref{lem:#1}}

\newcommand{\appref}[1]{Appendix~\ref{app:#1}}

\renewcommand{\P}{\mathbb{P}}
\newcommand{\E}{\mathbb{E}}

\newcommand{\reals}{\mathbb{R}}
\newcommand{\nats}{\mathbb{N}}

\DeclareMathOperator*{\argmin}{argmin}

\newcommand{\err}{\mathrm{err}}
\newcommand{\cA}{\mathcal{A}}

\newcommand{\cD}{\mathcal{D}}

\newcommand{\cF}{\mathcal{F}}

\newcommand{\cH}{\mathcal{H}}

\newcommand{\cJ}{\mathcal{J}}

\newcommand{\cL}{\mathcal{L}}

\newcommand{\cR}{\mathcal{R}}

\newcommand{\cX}{\mathcal{X}}
\newcommand{\cY}{\mathcal{Y}}

\usepackage{todonotes}

\begin{document}

\twocolumn[

\aistatstitle{On the Capacity Limits of Privileged ERM}

\aistatsauthor{ Michal Sharoni \And Sivan Sabato }

\aistatsaddress{ Department of Computer Science,\\
  Ben-Gurion University of the Negev\\ Beer-Sheva, Israel \And Department of Computer Science,\\
  Ben-Gurion University of the Negev\\ Beer-Sheva, Israel }]

\begin{abstract}
  We study the supervised learning paradigm called \emph{Learning Using Privileged Information}, first suggested by \cite{vapnik2009new}.
  In this paradigm, in addition to the examples and labels, additional (privileged) information is provided only for training examples. The goal is to use this information to improve the classification accuracy of the resulting classifier, where this classifier can only use the non-privileged information of new example instances to predict their label.
 We study the theory of privileged learning with the zero-one loss under the natural Privileged ERM algorithm proposed in \cite{pechyony2010theory}. We provide a counter example to a claim made in that work regarding the VC dimension of the loss class induced by this problem; We conclude that the claim is incorrect. We then provide a correct VC dimension analysis which gives both lower and upper bounds on the capacity of the Privileged ERM loss class.
We further show, via a generalization analysis, that worst-case guarantees for Privileged ERM cannot improve over standard non-privileged ERM, unless the capacity of the privileged information is similar or smaller to that of the non-privileged information. This result points to an important limitation of the Privileged ERM approach. In our closing discussion, we suggest another way in which Privileged ERM might still be helpful, even when the capacity of the privileged information is large.
  \end{abstract}

\section{INTRODUCTION}
\label{sec:intro}
The classical paradigm of supervised machine learning considers the following setting: given a set of labeled training examples, try to find in a given set of functions, the one with the smallest generalization error on the
unknown test examples. In this work,  we study an augmentation of this setting, first proposed by \cite{vapnik2009new}, called \emph{Learning Using Privileged Information}, or simply, \emph{privileged learning}.
In this paradigm, during the training stage, additional information about the training examples is provided to the learner. This information, called \emph{privileged information}, is available only for training examples during the training stage. The goal is to use this information to improve the classification accuracy of the resulting classifier. The classifier itself can use only non-privileged information of new examples to predict their label. Thus, the privileged information is only helpful inasmuch as it helps to obtain a better classifier.

A classical motivating example to this paradigm \citep[see][]{vapnik2009new} considers a case where the goal is to find a rule that predicts the outcome of a surgery after three months, based on information about the patient which is available before the surgery. However, for previous patients, there is additional information collected during and after the surgery. Although this information is not available during classification of new patients, it does exist in historical data and thus can be used as privileged information during training.

In this work, we study the natural ERM algorithm proposed in \cite{pechyony2010theory}, called Privileged ERM. This algorithm minimizes a joint loss of the non-privileged and the privileged information. We provide new results which point to the limitations of this approach when the privileged information is high-dimensional, or more generally, when the associated privileged loss class has a high capacity. High-dimensional privileged information is natural in many settings where offline measurements collected for training have a higher bandwidth or sensitivity than measurements during test time. For instance, consider a learning problem in which the goal is to classify images, in which the non-privileged information provides a low-resolution image, and the privileged information provides a high-resolution image. This would be the case if during training the training samples can be scanned using advanced equipment, while the classifier is deployed in a low-resource environment in the field, in which only low-quality images can be obtained. A similar application was studied in \cite{lee2020learning}.
We show here that the Privileged ERM approach with the zero-one loss cannot guarantee successes in this regime without additional assumptions.

We provide a VC dimension analysis for the loss class induced by the Privileged ERM algorithm. Our analysis includes a counter example to a claim previously made in \cite{pechyony2010theory}; The mistake can be traced to an error in the proof of that claim. We provide a correct analysis with both lower and upper bounds on the VC dimension. Thereafter, we study the regimes in which it is possible to provide a guarantee that Privileged ERM will result in an improved error bound over standard ERM, in which the privileged information is not used at all. We conclude that such worst-case guarantees must rely on a low-capacity privileged information class. Lastly, we suggest a possible way in which Privileged ERM can still be helpful, even when the capacity of the privileged information is large.

\section{RELATED WORK}
\label{sec:aChap}
The paradigm of privileged learning was first proposed by \cite{vapnik2009new}. This work introduced the SVM+ algorithm, which demonstrated how privileged information can be used in SVM-type algorithms, by changing their goal such that it will incorporate the privileged information.
In addition to introducing the SVM+ algorithm, \cite{vapnik2009new} derived results showing an improvement in the rate of convergence that can be achieved when utilizing privileged information in those types of algorithms, when the privileged-information class is low-dimensional.
\cite{pechyony2010theory} proposed an empirical risk minimization algorithm called Privileged ERM and generalized of the privileged learning optimization problem to other losses. They provide several theoretical claims regarding the convergence rate of this algorithm.

Since its inception, privileged learning has been applied in various domains. 
In \cite{lapin2014learning} the connection between SVM+ and weighted SVM is studied. It is shown that privileged information can be encoded by weights associated with every training example. In addition, it is shown that weighted SVM can always replicate an SVM+ solution, while the converse is not true.
In \cite{vapnik2015learning}, two mechanisms related to knowledge transfer between the instance space and the privileged information space are described. These mechanisms can be used for accelerating the speed of learning.
In \cite{qi2015semi}, a semi-supervised learning approach using privileged information is proposed. This approach can exploit both the distribution information in unlabeled data and privileged information, to improve the efficiency of the learning.
In \cite{yang2016empirical}, a metric-learning algorithm is proposed, which exploits privileged information to relax a previous method for metric-learning, under the ERM framework.
In \cite{vrigkas2016active}, a probabilistic approach is described, that combines learning using privileged information and active learning.
In \cite{pasunuri2016learning}, an algorithm for learning decision trees using privileged information is proposed. 
In \cite{vapnik2017knowledge}, a mechanism of knowledge transfer from the privileged information space to the features space is proposed. It is shown that this mechanism is applicable to a neural network framework as well as to SVM.
Recent works study privileged learning in vision domains \citep[e.g.,][]{yuan20193d,gao2019learning,li2019learning}. \cite{lee2020learning} considers an application in which the privileged information is high-dimensional. However, the theory of privileged learning has not addressed the capacity limits of privileged information under its basic methodologies.

\section{PRELIMINARIES AND SETTING}
\label{sec:setting}

We start by describing the privileged learning setting for general losses, as defined in \cite{pechyony2010theory}. Let  $\cX$ be the domain of elements that we wish to label. Let $\cX^*$ be the domain of the privileged information that is available for training examples. Let $\cY$ be the set of possible labels.
 The input to the learner consists of a sequence of i.i.d.~triplets:
 \begin{align}
&S=(x_1,x^*_1,y_1),...,(x_m,x^*_m,y_m), \notag \\ 
&x_i\in \cX,\quad x_i^*\in\cX^*,\quad y_i\in\cY,
\end{align}
generated according to a fixed but unknown probability distribution $\cD$ over $\cX \times \cX^* \times \cY$.
Let $\ell_{\cX}: \mathcal{Y} \times \mathcal{Y} \rightarrow \reals^+$ be a bounded loss function over the non-privileged example domain. The goal of privileged learning is to find a hypothesis that obtains a low loss on $\cD$, by using the sample $S$ that includes the privileged information.

Assume a bounded loss for privileged information, $\ell_{\cX^*}: \mathcal{Y} \times \mathcal{Y} \rightarrow \reals^+$. Let $C > 0$ be a constant, and denote $[t]_+ = \max(t, 0)$. Given a classifier $h:\cX \rightarrow \cY$ that uses only non-privileged information, and a privileged-information function $\phi:\cX^* \rightarrow \cY$, \cite{pechyony2010theory} define the loss of the composite hypothesis $(h, \phi)$ on the example $(x, x^*, y)$ by: 
\begin{align*}
&\ell'_C({h,\phi},(x,x^*,y))=  \\ 
&\quad\frac{1}{C}\ell_{\cX^*}(\phi(x^*),y)+[\ell_\cX(h(x),y)-\ell_{\cX^*}(\phi(x^*),y)]_+.
\end{align*}
The function $\phi$ is thought of as a ``correcting function'' for the loss induced by $h$ on the example.
Given a function class over the non-privileged information $\cH \subseteq \cY^\cX$, and a function class over the privileged information $\Phi\subseteq \cY^{\cX^*}$, \cite{pechyony2010theory} defined the Privileged ERM minimization problem as the following optimization problem:
\begin{equation}\label{eq:opterm}
\min_{h\in\cH, \phi\in \Phi} \sum\limits_{i=1}^m \ell'_C((h,\phi),(x_i,x_i^*,y_i)).
\end{equation}

In this work, we focus on the setting above in the important
special case of binary labels ($\cY = \{0,1\}$) and binary loss functions, with $C = 1$. In this
case, we have
\[
  \ell'_C(h,\phi, (x,x^*,y)) = \max\{\ell_{\cX}(h(x),y),\ell_{\cX^*}(\phi(x^*),y)\}.
\]
This leads to the following optimization problem:
 \begin{equation}\label{eq:ouropt}
\min_{h\in\cH, \phi\in \Phi} \sum\limits_{i=1}^m \max\{\ell_{\cX}(h(x_i),y_i), \ell_{\cX^*}(\phi(x^*_i),y_i) \}.
\end{equation}
It is instructive to think of $\phi$ as indicating which training examples should be taken into account when minimizing the loss over $h$, where $\phi(x^*_i)=1$ indicates that example $x_i$ should be ignored in the minimization. For instance, this could be relevant if the privileged information allows identifying the reliability of the labeling, as in a case of crowd-sourced labels. We thus assume that $\ell_\cX$ is the standard loss on the non-privileged information, defined by
 \[
   \ell^\zeroone_{\cX}(\hat{y},y):=\mathbf{1}[\hat{y} \neq y]
 \]
 and that $\ell_{\cX^*}$ is an ``ignoring'' loss on the privileged information, defined by:
 \[
   \ell^\ig_{\cX^*}(z,y):=\mathbf{1}[z=1].
 \]
Denote the error of $h$ with respect to the distribution $\cD$ by $\err(h,\cD) := \P_{(X,Y)\sim \cD}[h(X) \neq Y] = \E[\ell_\cX(h(X),Y)]$. Let $\err(h, S)$ be the empirical error of $h$ over the uniform distribution on $S$. The goal of privileged learning is thus to find a hypothesis from $\cH$ that obtains a low error on $\cD$, using the sample $S$. In the paradigm of Privileged ERM that we study here, this is attempted by solving the optimization problem in \eqref{ouropt}.

\section{VC-DIMENSION ANALYSIS}
\label{sec:vc}
In this section, we study the VC-dimension of the relevant function class for the minimization problem defined in \eqref{ouropt}. Denote the VC dimension of a class of functions by $\VC(\cdot)$. Denote $d := \VC(\cH)$ and $d^* := \VC(\Phi)$. 
Define
\[
  f_{(h,\phi)}((x,x^*),y) := \max(\ell_{\cX}^\zeroone(h(x),y),\ell^\ig_{\cX^*}(\phi(x^*),y)),
  \]
and let 
\begin{align*}
&\cF_{(\cH,\Phi)} =\{ f_{(h,\phi)} \mid h \in \cH, \phi \in \Phi\}.
\end{align*}
We write $\cF$ for $\cF_{(\cH,\Phi)}$ when the subscripts are clear from context.
\eqref{ouropt} is equivalent to running an ERM on $S$ with the hypothesis class  $\cF$.
Thus, the generalization behavior of Privileged ERM is characterized by the VC dimension of $\cF$.

What is the relationship between $\VC(\cF)$ and the values of $\VC(\cH), \VC(\Phi)$? This question was seemingly answered in \cite{pechyony2010theory}; They defined the following loss classes:
\begin{align*}
  \cL(\cH)&:=\{\ell_{\cX}(h(\cdot),\cdot)\mid h\in \cH \},\\
  \cL(\Phi)&:=\{\ell_{\cX^*}(\phi(\cdot),\cdot)\mid \phi\in \Phi \},\\
\cL(\cH,\Phi)&=\{\ell'_C((h,\phi),(\cdot,\cdot,\cdot))\mid h\in \cH, \phi \in \Phi \}.
\end{align*}
and claimed that the following equality holds:\footnote{The original claim includes real-valued losses, which requires generalizing the definition of VC-dimension; Here we state it for the special case of losses that map into $\{0,1\}$}
\begin{align}\label{eq:incorrect}
  &\text{ Claim of \cite{pechyony2010theory}: }\notag\\
  &\VC(\cL(\cH,\Phi))=\VC(\cL(\cH))+\VC(\cL(\Phi)).
\end{align}
The equality was then used to prove generalization upper bounds for the Privileged ERM optimization problem.

For $\ell_{\cX} := \ell^\zeroone_{\cX}$ and $\ell_{\cX^*} := \ell^\ig_{\cX^*}$, we have $\VC(\cH) = \VC(\cL(\cH))$, $\VC(\Phi) = \VC(\cL(\Phi))$ and $\VC(\cL(\cH,\Phi))=\VC(\cF_{(\cH,\Phi)})$. Therefore, if \eqref{incorrect} were true, it would imply that $\VC(\cF) = \VC(\cH) + \VC(\Phi)$.
However, we now show that this equality in fact \emph{does not hold}.\footnote{We traced the issue to an application of quantifiers in the wrong order in the proof of \eqref{incorrect} in \cite{pechyony2010theory}, which is available in the full version  \citep{pechyony2010theoryFull}.} 
\thmref{biggersum} below provides a counter example to the claimed \eqref{incorrect}.

\begin{theorem}\label{thm:biggersum}
For any integer $d > 0$, and any two domains $\cX,\cX^*$ such that $|\cX|,|\cX^*|\geq 3d$, there exist hypothesis classes $\cH_d\subseteq \{0,1\}^\cX$ and $\Phi_d \subseteq \{0,1\}^{\cX^*}$ such that $\VC(\cH_d)=\VC(\Phi_d)=d$ while $\mathrm{VC}(\cF_{(\cH_d,\Phi_d)}) = 3d$.
\end{theorem}
\begin{proof}
First, consider the case of $d=1$. Let $X_3 = \{x_1,x_2,x_3\} \subseteq \cX$ be a set of size three of domain examples from $\cX$. We describe a hypothesis $h$ from $X_3$ to $\{0,1\}$ via the triplet $(h(x_1),h(x_2),h(x_3))$. Similarly, let $X^*_3 = \{x^*_1,x^*_2,x^*_3\} \subseteq \cX^*$, and describe a hypothesis $\phi$ from $X_3^*$ to $\{0,1\}$ via the triplet $(\phi(x^*_1),\phi(x^*_2),\phi(x^*_3))$.
  Define the following hypothesis classes over $X_3$ and $X^*_3$:
  \begin{align*}
    \cH_1 &:= \{ (0,0,0), (0,0,1), (1,0,0), (1,1,0) \},\\
    \Phi_1 &:= \{(0,0,0), (0,0,1), (0,1,0), (1,0,1)\}.
  \end{align*}
It is easy to see that $\VC(\cH_1)=\VC(\Phi_1) = 1$.
On the other hand, when restricting $\cF_{(\cH_1,\Phi_1)}$ to the set $\widetilde{X}_3 := \{((x_i,x_i^*), 0)\}_{i \in [3]}$, we get that $\VC(\cF_{(\cH_1,\Phi_1)}) \leq 3$, since the domain is of size $3$. Moreover, the VC dimension is exactly $3$, since $\mathcal{F}_{(\mathcal{H}_1,\Phi_1)}$ induces all possible labelings on $\tilde{X}_3$:
Any labeling $h \in \mathcal{H}_1$ can be obtained by $f_{h,\phi_0}$, where $\phi_0$ is the all-zero function in $\Phi_1$. Similarly, all labelings in $\Phi_1$ can be obtained using the all-zero $h_0 \in \mathcal{H}_1$. The two additional missing labelings are $(0,1,1)$ and $(1,1,1)$. The first can be obtained using $h = (0,0,1)$ and $\phi = (0,1,0)$, and the second can be obtained using $h = (1,1,0)$ and $\phi = (0,0,1)$.
Thus, \mbox{$\VC(\cF_{(\cH_1,\Phi_1)}) = 3$}, as claimed. 

Next, consider $d>1$. Let $X_{3d} = \{x_1,...,x_{3d}\} \subseteq \cX$ be a set of $3d$ different domain points from $\cX$. Partition these points into $d$ triplets, denoted $t_1:=(x_1,x_2,x_3),...,t_d:=(x_{3d-2},x_{3d-1},x_{3d})$. We describe a hypothesis $h$ over $X_{3d}$ as a sequence of $d$ functions from $\cH_1$ that are applied to the examples in the triplets $t_1,\ldots,t_d$. A description of $\phi$ over $X_{3d}^*$ is analogous. Define the following hypothesis classes of functions from $X_{3d}$ to $\{0,1\}$ and from $X^*_{3d}$ to $\{0,1\}$:
  \begin{align*}
    \cH_d &:= \{(h_1, ... ,h_d) \mid h_1,...,h_d \in \cH_1\},\\
    \Phi_d &:= \{ (\phi_1, ... , \phi_d) \mid \phi_1,...,\phi_d \in \Phi_1 \}.
  \end{align*}
  We first prove that $\VC(\cH_d)=d$. Suppose for contradiction that $\VC(\cH_d)>d$. Then there exists a shattered set with $d+1$ points. Since the predictors are defined on $d$ triplets, by the pigeonhole principle, there must be two points of the shattered set that are from the same triplet. From this we can conclude that $\cH_1$ shatters a set of size two, in contradiction to $\VC(\cH_1)=1$. Therefore, $\VC(\cH_d)\leq d$.
  Next, we prove that $\VC(\cH)\geq d$, by showing that there exists a shattered set of size $d$. Since $\VC(\cH_1)=1$, in each of the $d$ domain triplets there is a point that $\cH_1$ shatters. The set of all of these points is a set of size $d$ which is shattered by $\cH_d$, as needed. 
We conclude that $\VC(\cH_d)=d$. By analogous arguments, $\VC(\Phi_d)=d$. 
 
Lastly, we show that \mbox{$\VC(\cF_{(\cH_d,\Phi_d)})=3d$}. Define the set
\[
  \cF_d:=\{(f_1, ... ,f_d) \mid f_1,...,f_d \in \cF_{(\cH_1,\Phi_1)}\}.
  \]
 We first claim that $\cF_d \subseteq \cF_{(\cH_d,\Phi_d)}$:
 Let \mbox{$(f_{(h_1,\phi_1)},...,f_{(h_d,\phi_d)}) \in \cF_d$}. From the definition of $\cF_{(\cH_1,\Phi_1)}$, $\forall i \in [d]$ we have
 \[
   f_{(h_i,\phi_i)}((x,x^*), y)= \max(\ell_{\cX}^\zeroone(h_i(x),y),\ell^\ig_{\cX^*}(\phi_i(x^*),y)).
   \]
   Therefore,
   \[
     (f_{(h_1,\phi_1)},...,f_{(h_d,\phi_d)}) = f_{((h_1,..., h_d) ,(\phi_1,..., \phi_d))} \in \cF_{(\cH_d,\Phi_d)}.
   \]
From this we conclude that 
$
 \mathrm{VC}( \cF_{(\cH_d,\Phi_d)}) \geq \mathrm{VC}(\cF_d).
$
Now, consider the set $\widetilde{X}_{3d} = \{((x_i,x_i^*), 0)\}_{i \in [3d]}$. 
Restricting $\cF_d$ to the set $\widetilde{X}_{3d}$ results in the set of all possible functions over each triplet in $\widetilde{X}_{3d}$, as in the case of $d=1$. Thus, $\cF_d$ is shattered by $\widetilde{X}_{3d}$. Therefore, $\mathrm{VC}(\cF_d)\geq 3d$ as needed.
\end{proof}

We provide a correct upper bound for $\mathrm{VC}(\cF_{(\cH,\Phi)})$ in the following theorem.
\begin{theorem}\label{thm:vcupper}
Let $d,d^*$ be integers and let $\cX$ be some domain. Let $\cH \subseteq \{0,1\}^\cX, \Phi \subseteq \{0,1\}^{\cX^*}$ be hypothesis classes such that $\VC(\cH)=d$ and $\VC(\Phi)=d^*$. Then
\[
\mathrm{VC}(\cF_{(\cH,\Phi)})\leq 4 \log_2 (4e)(d+ d^* +1) \approx 13.77 (d+ d^* +1).
\]
\end{theorem}

To prove the theorem, we first provide a tight upper bound on the VC dimension of the union of two hypothesis classes over the same domain.
\begin{lemma}\label{lem:union}
  Let $\cY = \{0,1\}$. Let $\cJ$, $\cH \subseteq \cY^\cX$ be two hypothesis class over the same domain. Then
  \[
    \mathrm{VC}(\cH \cup \cJ) \leq \VC(\cH)+\VC(\cJ)+1.
    \]

  This bound is tight: For any $d, d^* \in \mathbb{N}$, there exist $\cJ$ and $\cH$ such that $\VC(\cH)=d$, $\VC(\cJ)=d^*$ and $\mathrm{VC}(\cH \cup \cJ) = d+d^*+1.$
\end{lemma}
\begin{proof}
  For a hypothesis class $\cF$, denote the growth function of $\cF$ by
  \[
    \Pi_{\cF}(m):= \max\{\, |\cF|_S| \mid S\subseteq \cX \times \cY , |S|=m\,\}.
  \]
  Clearly, for any $S$, $|\cH \cup \cJ|_S \leq |\cH|_S + |\cJ|_S$. Therefore $\Pi_{\cH \cup \cJ}(m)\leq \Pi_{\cH}(m)+ \Pi_{\cJ}(m).$
By Sauer's Lemma \citep{sauer1972density}  and using the identity $\binom{m}{k}=\binom{m}{m-k}$, denoting $d := \VC(\cH)$ and $d^* := \VC(\cJ)$, we get
\begin{align}\label{eq:eqq1}
&\Pi_{\cH \cup \cJ}(m)\leq \sum\limits_{i=0}^{d}\binom{m}{i}+ \sum\limits_{i=0}^{d^*}\binom{m}{i}  \notag \\
&=\sum\limits_{i=0}^{d}\binom{m}{i}+ \sum\limits_{i=m- d^*}^{m}\binom{m}{i}.
\end{align}
If $m > d + d^* + 1$ then $m-d^* \geq d+2$, so:
\begin{align}\label{eq:eqq2}
&\sum\limits_{i=0}^{d}\binom{m}{i}+ \sum\limits_{i=m- d^*}^{m}\binom{m}{i}  \leq \sum\limits_{i=0}^{m}\binom{m}{i} = 2^m.
\end{align}
Combining \eqref{eqq1} and \eqref{eqq2}, we conclude that for $m > d + d^* + 1$, $\Pi_{\cH \cup \cJ}(m)< 2^m$. Therefore,
\[
  \mathrm{VC}(\cH \cup \cJ) \leq d + d^* + 1.
\]
This proves the upper bound.

To show that this bound is tight, let $d, d^* \in \mathbb{N}$ and consider a
domain $\cX$ of size $d+ d^*+1$. Let $\cH$ be the set of all the functions
that map at most $d$ examples from $\cX$ to 1 and let $\cJ$ be the set of all the
functions that map at most $d^*$ examples from $\cX$ 
to 0. Then, $\VC(\cH)=d$ and $\VC(\cJ)=d^*$. To show that
$\mathrm{VC}(\cH \cup \cJ) \geq d+d^*+1$, consider two cases for a labeling of $\cX$:
\begin{itemize}
\item If the labeling includes at most $d$ positive labels, then there is a function in $\cH$ that provides this labeling.
\item  If the labeling includes more than $d$ positive labels, then since the domain is of size  $d+ d^*+1$, the labeling contains at most $d^*$ negative labels. Therefore, there is a function in $\cJ$ that provides this labeling.
\end{itemize}
Thus, $\cH \cup \cJ = \cY^\cX$, hence $\mathrm{VC}(\cH \cup \cJ) = |\cX| = d+d^*+1$, as claimed.
\end{proof}

In our proof of \thmref{vcupper} we further use a theorem from \cite{blumer1989learnability}.\footnote{In
 \cite{blumer1989learnability}, the dependence on $k$ was not specified in the theorem statement; we extracted the exact constants from the proof.}

\begin{theorem}[\citealt{blumer1989learnability}]\label{thm:cite1}
Let $(\cX, \cR)$ be a set system,  where $\cX$ is a set of elements and $\cR$ is a set of subsets of $\cX$. For an integer $k \geq 2$ and a set of sets $\cR$, define $\cR^{k \cup}:=\{ R_1 \cup...\cup R_k \mid R_1,...,R_k \in \cR  \}$, the k-fold union of $\cR$. Then
\begin{align*}
&\mathrm{VC}(\cR^{k \cup}) \leq \mathrm{VC}(\cR)\cdot 2k  \log_2(2ek).
\end{align*}
\end{theorem}
We now prove \thmref{vcupper}.
\begin{proof}[Proof of \thmref{vcupper}]
  Let $\cH, \Phi$ be hypothesis classes as stated in the theorem.
  Define $\cH', \Phi' \subseteq \{0,1\}^{\cX \times \cX^* \times \{0,1\}}$ as follows:
  \[
    \cH' := \{ (x,x^*,y) \mapsto \ell_{\cX}^\zeroone(h(x),y) \mid h \in \cH\}\]
  and
  \[
    \Phi' := \{ (x,x^*,y) \mapsto \ell^\ig_{\cX^*}(\phi(x^*),y) \mid \phi \in \Phi\}.
  \]
  In addition, let
  \[\cF'_{\cH,\Phi} := \{ (x,x^*,y) \mapsto f_{(h,\phi)}((x,x^*),y) \mid h\in \cH, \phi \in \Phi\}.
      \]
      It is easy to see that $\VC(\cH') = \VC(\cH)$, $\VC(\Phi') = \VC(\Phi)$, and $\VC(\cF'_{\cH,\Phi})  = \VC(\cF_{\cH,\Phi})$.
  In addition, for any $h \in \cH, \phi \in \Phi$, we have
  \begin{align*}
    & \{ (x,x^*,y) \mid f_{(h,\phi)}((x,x^*),y) = 1\}  =\\
    &\qquad\{ (x,x^*,y) \mid \ell_{\cX}^\zeroone(h(x),y) = 1 \} \\
    &\qquad\cup \{ (x,x^*,y) \mid \ell^\ig_{\cX^*}(\phi(x^*),y) = 1\}.
  \end{align*}
 
  Therefore, treating functions in $\{0,1\}^{\cX \times \cX^* \times \{0,1\}}$ as sets, we have $f_{(h,\phi)} = h' \cup \phi'$, where $h' \in \cH'$ and $\phi' \in \Phi'$. Letting $\cR:= \cH' \cup \Phi'$, it follows that $\cF_{(\cH,\Phi)} \subseteq \cR^{2 \cup}$. Therefore, $\mathrm{VC}(\cF) \leq \mathrm{VC}(\cR^{2 \cup}).$
  By \thmref{cite1} with $k =2$, it follows that $\mathrm{VC}(\cF)\leq \mathrm{VC}(\cR)\cdot 4  \log_2(4e)$. In addition, by \lemref{union},
$
\mathrm{VC}(\cR) = \mathrm{VC}(\cH \cup \Phi) \leq d+ d^* +1
$.
Therefore,
$\mathrm{VC}(\cF)\leq 4 \log_2 (4e)(d+ d^* +1),$
as claimed.
\end{proof}

\section{CAPACITY LIMITS OF PRIVILEGED ERM}
\label{sec:error}
We now turn to study the convergence rate of the solution to the optimization problem in \eqref{ouropt}, and derive conditions on the VC dimension values that allow this bound to be better than the known bound for regular ERM.
We show that for guaranteed generalization improvement, the VC dimension of the privileged information cannot be much larger than the VC dimension of the non-privileged information, thus limiting the usefulness of this approach in the case of zero-one losses. 

A guaranteed generalization improvement occurs if the error guarantee of the optimization problem in \eqref{ouropt} is smaller than that of standard ERM generalization bounds. As shown in previous works \cite[e.g., ][]{vapnik2009new}, this requires bounds that take into account the error with respect to the hypothesis class: The advantage of privileged information, when it exists, comes from the possibility of faster convergence due to a smaller error rate. \cite{boucheron2005theory} provide tight error bounds that take into account the error.
Fixing $m \in \nats$, denote for $d \in \nats, x \in [0,1]$,
\begin{align*}
  &\nsqrt(d):= \frac{8 d \log(m+1) + 4\log(\frac{4}{\delta})}{m} \\
  &\wsqrt(x,d):= \sqrt{x \cdot \nsqrt(d)},
\end{align*}
where $\nsqrt$ stands for a fast rate an $\wsqrt$ stands for a slow rate.
By \citet[Corollary 5.2]{boucheron2005theory}, denoting by $\ermres$ the output of a standard ERM algorithm for the hypothesis class $\cH$ on the sample $S$, and its training error by $\hat{\epsilon}_{\mathrm{ERM}}:=\err(\hat{h}_{\mathrm{ERM}}, S)$, we have
that with a probability at least $1-\delta$, 
\begin{align*}
  &\err(\hat{h}_{\mathrm{ERM}}, \cD)\leq  \hat{\epsilon}_{\mathrm{ERM}} + \wsqrt(\hat{\epsilon}_{\mathrm{ERM}},d) + \nsqrt(d) := \ermbound.
\end{align*}
Based on this result, we derive an analogous upper bound for the case of Privileged ERM. To provide the bound, we first define some notations. 
Define an auxiliary loss function
\[
  \ell^a_{(h,\phi)}(x,x^*,y) := \mathbf{1}[h(x) \neq y \wedge \phi(x^*) = 0]
\]
and the loss class $\cL^a_{(\cH,\Phi)}:=\{\ell^a_{(h,\phi)} \mid h\in\cH, \phi\in\Phi\}$. Denote $d_a := \VC(\cL^a_{(\cH,\Phi)})$. By \thmref{vcupper}, $d_a \leq 13.77(d+d^*+1)$. The following lemma provides a lower bound for $d_a$, leading to the conclusion that $d_a = \Theta(d+d^*)$.
\begin{lemma}\label{lem:vcF}
For $\cH$, $\Phi$ such that $d,d^* > 1$,
\[
d_a \geq d+d^*- 2.
\]
\end{lemma}
\begin{proof}
Let $C_{\cH}=\{x_1,...,x_d\} \subseteq \cX$ be a set of size $d$ that is shattered by $\cH$ and $C_{\Phi}=\{x^*_1,...,x^*_{d^*}\} \subseteq \cX^*$ be a set of size $d^*$ that is shattered by $\Phi$.
Define
\begin{align*}
  &C_1 := \{(x_i,x_{d^*}^*,0) \mid 1\leq i \leq d-1 \},\\
  &C_2:=\{(x_d,x_j^*,0) \mid 1 \leq j \leq d^*-1 \},\\
  &C_{\cL^a_{(\cH,\Phi)}}:= C_1 \cup C_2.
\end{align*}
Note that $|C_{\cL^a_{(\cH,\Phi)}}|=|C_1|+|C_2|= d + d^*-2$. We now show that $C_{\cL^a_{(\cH,\Phi)}}$ is shattered by $\cL^a$.

Let $L=(l_1,...,l_{d+d^*-2}) \in \{0,1\}^{d+d^*-2}$ be a potential labeling of $\cL^a$.
For every $p=(x,x_{d^*}^*,0)\in C_1$, let $l(p)$ be the label of $p$ according to $L$. Let $h \in \cH$ be such that for every $p=(x,x_{d^*}^*,0)\in C_1$,  $h(x) = l(p)$, and also $h(x_d)=1$. Such an $h\in \cH$ exists since  $C_{\cH}$ is shattered by $\cH$.
Similarly, for every $p^*=(x_d,x^*,0)\in C_2$, let $l(p^*)$ be the label of $p^*$ according to $L$. Let $\phi \in \Phi$ be such that for every $p^*=(x_d,x^*,0)\in C_2$, $\phi(x^*)=1-l(p^*)$, and also $\phi(x^*_{d^*})=0$. Such a $\phi\in \Phi$ exists since $C_{\Phi}$ is shattered by $\Phi$. 
We now claim that  $\ell^a_{(h,\phi)}$ obtains the labeling $L$ for $C_{\cL^a_{(\cH,\Phi)}}$: For every $p=(x,x_{d^*}^*,0)\in C_1$,
\begin{align*}
  &\ell^a_{(h,\phi)}(x,x^*_{d^*},0) = \mathbf{1}[h(x) \neq 0 \wedge \phi(x^*_{d^*}) = 0]\\
  &=\mathbf{1}[h(x) \neq 0]=l(p).
\end{align*}
  In addition, for every $p^*=(x_d,x^*,0)\in C_2$,
  \begin{align*}
    &\ell^a_{(h,\phi)}(x_d,x^*,0)  = \mathbf{1}[h(x_d) \neq 0 \wedge \phi(x^*) = 0]\\
    &=\mathbf{1}[\phi(x^*) = 0]=l(p^*).
  \end{align*}

We conclude that $C_{\cL^a_{(\cH,\Phi)}}$ is shattered by $\cL^a$. Since $|C_{\cL^a_{(\cH,\Phi)}}|= d + d^*-2$, $\mathrm{VC}(\cL^a_{(\cH,\Phi)}) \geq d + d^* - 2$ as claimed.
\end{proof}

Denote by $(\hat{h},\hat{\phi})$ some assignment that obtains the minimum of the optimization problem in \eqref{ouropt}. 
Denote the \emph{empirical ignored weight} by
\[
  \widehat{\epsilon}_{ig}:=\frac{1}{m} \sum\limits_{(x,x^*,y)\in S }\mathbf{1}[\hat{\phi}(x^*)=1].
\]
This is the fraction of training examples that are ignored due to the privileged information when minimizing the error over $\cH$.  Denote the \emph{empirical unexplained error} by
\[
  \widehat{\epsilon}_{u}:=\frac{1}{m} \sum\limits_{(x,x^*,y)\in S }\mathbf{1}[\ell^a_{(\hat{h},\hat{\phi})}(x,x^*,y)=1].
\]
This is the fraction of training examples that were not ignored by $\phi$, but were still classified incorrectly by $\hat{h}$. The following theorem gives a generalization error bound for Privileged ERM.

\begin{theorem}\label{thm:firstb}
With a probability $1-2\delta$ over the random choice of $S \sim \cD^m$, 
\begin{align*}
  &\err(\hat{h},\cD) \leq \widehat{\epsilon}_{ig} + \widehat{\epsilon}_{u}  + \wsqrt(\widehat{\epsilon}_{ig}, d^*)+ \wsqrt(\widehat{\epsilon}_{u}, d_a) \\
  & +\nsqrt(d^*)+ \nsqrt(d_a):= \ourbound.
\end{align*}
\end{theorem}
\begin{proof}
  Let $\cD'$ be a distribution over \mbox{$(\cX \times \cX^* \times \cY) \times \{0\}$} such that the marginal over $(\cX \times \cX^* \times \cY)$ is $\cD$.
  Recall that $(\hat{h},\hat{\phi})$ are minimizers of \eqref{ouropt}. Decompose the error of $\hat{h}$ as follows:
 \begin{align}
  &\err(\hat{h},\cD) = \notag \\
   & =\P[\hat{h}(X) \neq Y \wedge \hat{\phi}(X^*)=1] \notag\\
   &\quad+\P[\hat{h}(X) \neq Y \wedge  \hat{\phi}(X^*)=0] \notag\\
   & \leq \P[\hat{\phi}(X^*)=1]+\P[\hat{h}(X) \neq Y \wedge  \hat{\phi}(X^*)=0]\notag\\
   & = \err(\hat{\phi},\cD')+\err(\ell^a_{(\hat{h},\hat{\phi})},\cD'),\label{eq:r1}
 \end{align}
 Where we treat $\hat{\phi}$ as equivalent to $(x,x^*,y) \mapsto \hat{\phi}(x^*)$.
 
We will bound each of the terms on the RHS separately.
 
 Given $S= (\,(x_i,x_i^*,y_i)\,)_{i \in [m]}\sim \cD^m$, let
$S_0:=(\,((x_i,x_i^*,y_i),0)\,)_{i \in [m]}$, so that $S_0$ is distributed as an i.i.d.~sample from $\cD'$. Then $\err(\ell^a_{(\hat{h},\hat{\phi})},S_0) = \widehat{\epsilon}_{u}$ and $\err(\hat{\phi},S_0) = \widehat{\epsilon}_{ig}$.

By \cite{boucheron2005theory}, Given a sample $\tilde{S}=((a_1,b_1),\dots,(a_m,b_m))$ generated according to a distribution $\tilde{\cD}$ over $\cA\times \{0,1\}$ and a hypothesis class $\cJ \subseteq \{0,1\}^\cA$ with VC dimension $d'$, with probability at least $1-\delta$, for all $g \in \cJ$, if $\hat{\epsilon} =\err(g,\tilde{S})$, then  
 \begin{align}
   &\err(g,\tilde{\cD})
 \leq \hat{\epsilon} + \wsqrt(\hat{\epsilon}) + \nsqrt(d').\label{eq:p5}
 \end{align}

Assigning $\tilde{\cD}:=\cD'$, $\cJ:=\cL_{\cH,\Phi}^a$, $\tilde{S}:=S_0$, it follows that 
 \begin{align}
   &\err(\ell^a_{(\hat{h},\hat{\phi})},\cD')\! \leq\! \widehat{\epsilon}_{u}  + \wsqrt(\widehat{\epsilon}_{u},d_a) + \nsqrt(d_a).\label{eq:p6}
 \end{align}
 In addition, assigning $\tilde{\cD}:=\cD'$, $\cF:=\Phi$, $\tilde{S}:=S_0$, \eqref{p5}, we get 
 \begin{equation}\label{eq:r2}
 \err(\hat{\phi}, \cD') \leq \widehat{\epsilon}_{ig} + \wsqrt(\widehat{\epsilon}_{ig}, d^*) + \nsqrt(d^*).
 \end{equation}
 Combining \eqref{p6} and \eqref{r2} with \eqref{r1} and using the union bound,  with probability at least $1-2\delta$, 
 \begin{align*}
  &\err(\hat{h},\cD) \leq \widehat{\epsilon}_{ig} + \wsqrt(\widehat{\epsilon}_{ig}, d^*) + \nsqrt(d^*) \\
   &\qquad\qquad + \widehat{\epsilon}_{u} + \wsqrt(\widehat{\epsilon}_{u}, d_a) + \nsqrt(d_a),
 \end{align*}
 as claimed.
\end{proof}
The upper bound in \thmref{firstb} is derived using known upper bounds for ERM under bounded agnostic error. While these upper bounds are known to be tight for the zero-one loss, there does not exist an equivalent result for the loss we use for the predictions of $\Phi$. The following theorem shows that nonetheless, the classical agnostic uniform convergence upper bound is tight also for this loss. The proof is provided in \appref{lower}.

\begin{theorem}\label{thm:lower}
Let $\Phi$ be a hypothesis class with $\VC(\Phi)=d^*$. For all $\epsilon \in (0,1)$ and $\delta<1/128,$
if the sample size is $m < (d^*-1)/(1280 \cdot \epsilon^2),$
then there exists a distribution $\cD$ such that with a probability larger than $\delta$, 
$\exists \phi \in \Phi\text{ such that }\P[\phi(X)=1]-\hat{\P}[\phi(X) = 1]>\epsilon$,
where $\hat{\P}$ denotes the empirical probability based on a random i.i.d.~sample of size $m$.  
\end{theorem}

We wish to derive conditions under which \mbox{$\ourbound < \ermbound$}. However, each of these bounds uses different empirical measures. Our next lemma links the two sets of measures, by showing that regardless of the set of examples that are ignored, the empirical error of an ERM algorithm is smaller than the value of the minimization of  \eqref{ouropt}. 
\begin{lemma}\label{lem:difficulty}
For any $S' \subseteq S$ and for any $h \in \cH$, 
\[
m\cdot \hat{\epsilon}_{\mathrm{ERM}} \leq \sum\limits_{(x_i,y_i)\in S \backslash S' }\mathbf{1}[h(x_i) \neq y_i]+ |S'|. 
\]
It follows that $\hat{\epsilon}_{\mathrm{ERM}} \leq  \widehat{\epsilon}_{u} + \widehat{\epsilon}_{ig}.$
\end{lemma}
\begin{proof}
Let $h \in \cH$ and  $S' \subseteq S$. By definition, we have $\hat{\epsilon}_{\mathrm{ERM}} \leq \err(h,S)$. In addition,
\begin{align*}
&m\cdot \err(h,S)  = \sum\limits_{(x_i,y_i)\in S }\mathbf{1}[h(x_i) \neq y_i]  \\
& =  \sum\limits_{(x_i,y_i)\in S \backslash S' }\mathbf{1}[h(x_i) \neq y_i] + \sum\limits_{(x_i,y_i)\in S' }\mathbf{1}[h(x_i) \neq y_i]\\
  & \leq \sum\limits_{(x_i,y_i)\in S \backslash S' }\mathbf{1}[h(x_i) \neq y_i]+|S'|.
\end{align*}
This proves the claim.
\end{proof}
This lemma is crucial for the comparison of $\ourbound$ and $\ermbound$, as it implies that the only way to get $\ourbound < \ermbound$ is to have smaller convergence terms in $\ourbound$ compared to $\ermbound$. 

Next, we derive a sufficient condition for having $\ourbound \leq \ermbound$.
This condition considers is a best-case scenario, in the sense that it requires the privileged information to cause the ERM to ignore exactly the examples that cannot be classified correctly using $\cH$. Under this scenario, the privileged learning bound is smaller than the ERM bound if the unexplained error is sufficiently small. This theorem is proved in \appref{boundssecond}.
\begin{theorem}\label{thm:boundssecond}
Suppose that $\hat{\epsilon}_{\mathrm{ERM}} = \widehat{\epsilon}_{ig} + \widehat{\epsilon}_{u}$. Assume that $\delta$ is fixed and $d,d^*,d_a$ are large. Then,  if 
\begin{align*}
&\sqrt{\widehat{\epsilon}_{u}} \leq \\
& \sqrt{\hat{\epsilon}_{\mathrm{ERM}}}\cdot\Theta(\frac{\sqrt{d}-\sqrt{d^*}}{\sqrt{d_a}})-\sqrt{\frac{\log(m)}{m}}\cdot\Theta(\frac{d_a+d^*-d}{\sqrt{d_a}}),
\end{align*}
then 
$\ourbound \leq \ermbound$.
\end{theorem}
Note that by \lemref{vcF}, the second term is necessarily positive.

The sufficient condition above is stricter when the empirical error is smaller. In particular, in the realizable case, where $\hat{\epsilon}_{\mathrm{ERM}}=0$, this sufficient condition never holds. In addition, the sufficient condition can only hold if $d^*\leq d$ (in addition to a small sample size). Therefore, this does not allow a privileged class $\Phi$ of a large capacity. Indeed, the following result shows that in general, $d^*$ cannot be much larger than $d$ while still allowing $\ourbound \leq \ermbound$.
\begin{theorem}\label{thm:secondalpha}
For any fixed $\delta \in (0,1)$, if $\ourbound \leq \ermbound$ then
 \[
d^* \leq 2.25 \cdot d + o(1).
\]
The convergence of the last term is with respect to the growth of $d,d^*$ together and/or of $m$.
\end{theorem}
This theorem is proved in \appref{secondalpha}. 

\section{DISCUSSION}

Our work shows that the Privileged ERM approach for privileged learning
suffers an inherent capacity limit on the privileged information class in the
case of the zero-one loss. This analysis is relevant when ERM can be
accurately executed and pertains to the \emph{statistical}
benefits of privileged learning. However, when surrogate losses are used, the
situation may be quite different. In these cases, privileged learning may have
a \emph{computational} advantage, as it may be possible to use privileged
information to make the computational problem of minimizing the true loss
easier. For instance, if privileged information allows identifying outliers, and thus helps to ignore some
training examples in a way that would make the optimization objective of the
surrogate loss closer to that of the true target loss, the resulting training
error could be lower, leading to a lower true error. We plan to study this
promising direction in future work.

We further note that our analysis only provides a limitation on the dimension of the privileged information under worst-case analysis and within a specific privileged-ERM framework. Studying other variants of this framework may lead to less restrictive results.

\bibliographystyle{plainnat}
\bibliography{privileged}

\appendix
\onecolumn
\section{DEFERRED PROOFS}
Below, we provide deferred proofs for theorems stated above.
\appref{lower} provides the proof of the lower bound, \thmref{lower}, \appref{boundssecond} provides the proof of the sufficient condition, \thmref{boundssecond} and \appref{secondalpha} provides the proof of the necessary condition, \thmref{secondalpha}. 

\subsection{Proof Of The Lower Bound For The Privileged Learning Loss}
\label{app:lower}
\begin{proof}[Proof of \thmref{lower}]
  This proof is an adaptation of the proof of the lower bound for the zero-one loss given in
  \citet[Theorem 5.2]{anthony2009neural} to our setting. The main challenge is constructing a set of distributions that can only be distinguished using a worst-case number of samples. This is achieved using the following new construction.

Since $\Phi$ has VC-dimension $d^*$ , there is a set $C = \{x_1^*,..., x_{d^*}^*\}$ of $d^*$
examples that is shattered by $\Phi$. For simplicity, assume that $d^*$ is an even number. If $d^*$ is odd, then the proof below holds for $d^*-1$.
We partition the set $C$ into $d^*/2$ pairs $\{(a_i,b_i)\}_{i \in [d^*/2]}$. 
Consider the class of all distributions $\cD$ over $\cX \times \cY$ with the following properties:
\begin{itemize}
\item $\cD$ assigns a zero probability to all sets not intersecting $C \times \{0,1\}$.
\item For $x \in \cX$, denote $\cD(x) = \P_{(X,Y)\sim \cD}[X = x]$. Set $\alpha:=\frac{8\epsilon}{(1-8\delta)}$. For each $i = 1,2,...,d^*/2$ and a pair $(a_i,b_i)$ in the partition of  $C$, either:
\begin{itemize}
\item $\cD(a_i)=\frac{1+\alpha}{d^*}$ and $\cD(b_i)=\frac{1-\alpha}{d^*}$, or
\item $\cD(b_i)=\frac{1+\alpha}{d^*}$ and $\cD(a_i)=\frac{1-\alpha}{d^*}$.
\end{itemize}

\end{itemize}
Let $\Phi' \subseteq \Phi$ be the set including all hypotheses $\phi$ such that for each pair $(a_i,b_i)$, $\phi$ maps one of the elements in the pair to $0$ and the other to $1$.

Given $\cD$, Let $\phi^* \in \Phi'$ be the function such that for each pair $(a_i,b_i)$,  $\phi^*(a_i)=1$ if and only if $\cD(a_i)=\frac{1-\alpha}{d^*}$. Then,
\[
\P[\phi^*(X)=1]= \sum\limits_{i=1}^{d^*/2}\frac{1-\alpha}{d^*} = \frac{1-\alpha}{2}.
\] 
Furthermore, for any $\phi \in \Phi'$, we have
\begin{align*}
\P[\phi(X)=1]
=  \sum\limits_{i=1}^{d^*/2}(\frac{1+\alpha}{d^*}\mathbf{1}[\phi(a_i) \neq \phi^*(a_i)] + \frac{1-\alpha}{d^*}\mathbf{1}[\phi(a_i) =\phi^*(a_i)]) 
 = \P[\phi^*=1] +\frac{2\alpha}{d^*} \sum\limits_{i=1}^{d^*/2}\mathbf{1}[\phi(a_i) \neq \phi^*(a_i)].
\end{align*}

For any sample $S \in S^m$, let $N(S)=(N_1(S),...,N_{d^*/2}(S))$, where $N_i(S)$ is the number of occurrences of either $a_i$ or $b_i$ in $S$.  Then, letting $L$ be a learning algorithm for $\Phi'$, we have that for $\hat{\phi}:=L(S)$,
\begin{align*}
\E[\frac{2}{d^*} \sum\limits_{i=1}^{d^*/2}\mathbf{1}[\hat{\phi}(a_i) \neq \phi^*(a_i)]] 
&= \frac{2}{d^*} \E[\sum\limits_{i=1}^{d^*/2}\mathbf{1}[\hat{\phi}(a_i) \neq \phi^*(a_i)]]\\
& = \frac{2}{d^*} \sum\limits_{N} \sum\limits_{i=1}^{d^*/2}\P[\hat{\phi}(a_i) \neq \phi^*(a_i)\mid N(S)=N]\cdot \P[N(S)=N].
\end{align*}
where $N = (N_1,..., N_{d^*/2})$ ranges over the set of $d^*/2$-tuples of positive integers with $\sum\limits_{i=1}^{d^*/2}N_i=m$. 

Similarly to the proof of \citet[Theorem 5.2]{anthony2009neural}, we can conclude that if
$
m < \frac{d^*}{320 \cdot \epsilon^2},
$
then with a probability larger than $1/64$ over samples $S \sim \cD^m$, 
$
\P[\hat{\phi}(X)=1]-\P[\phi^*(X)=1]>\epsilon.
$
In particular, this holds for $\hat{\phi} = \argmin_{\phi \in \Phi'} \hat{\P}[\phi(X)= 1]$.

Let $m < \frac{d^*}{1280 \cdot \epsilon^2}$. By the conclusion above, we have that with a probability larger than $\delta$ over samples $S \sim \cD^m$, $\P[\hat{\phi}(X)=1]-\P[\phi^*(X)=1]>2\epsilon.$

We now claim that at least one of the following holds with a probability larger than $\delta$: 
\begin{itemize}
\item $|\P[\hat{\phi}(X)=1]-\hat{\P}[\hat{\phi}(X)=1]|>\epsilon$;
\item $|\P[\phi^*(X)=1]-\hat{\P}[\phi^*(X)=1]|>\epsilon$.
\end{itemize}

Assume in contradiction that each of these inequalities holds with a probability at most $\delta$.
Then, with a probability at least $1-2\delta$,
\[\P[\hat{\phi}(X)=1]-\epsilon<\hat{\P}[\hat{\phi}(X)=1\]
and
\[\hat{\P}[\phi^*(X)=1]<\P[\phi^*(X)=1]+\epsilon.\]
Also, from the definition of  $\hat{\phi}$, we have $\hat{\P}[\hat{\phi}(X)=1] \leq \hat{\P}[\phi^*(X)=1].$ We get that with a probability at least $1-2\delta$, $\P[\hat{\phi}(X)=1]-\P[\phi^*(X)=1]<2\epsilon$. Since $\delta < 1/128$ and $m < \frac{d^*}{1280 \cdot \epsilon^2}$, this contradicts the lower bound above. 
It follows that at least one of the assumed inequalities above holds, which proves the claim. 
\end{proof}

\subsection{Proof Of The Sufficient Condition}
\label{app:boundssecond}
We now prove \thmref{boundssecond}.  
We derive a sufficient condition for the following inequality to hold:
\begin{align*}
  \ourbound =\widehat{\epsilon}_{ig} + \widehat{\epsilon}_{u} + \wsqrt(\widehat{\epsilon}_{ig}, d^*)
   + \wsqrt(\widehat{\epsilon}_{u}, d_a) +  \nsqrt(d^*) + \nsqrt(d_a)  
\leq \hat{\epsilon}_{\mathrm{ERM}} + \wsqrt(\hat{\epsilon}_{\mathrm{ERM}},d) + \nsqrt(d) = \ermbound.
\end{align*}
Under the assumption that 
$\hat{\epsilon}_{\mathrm{ERM}}= \widehat{\epsilon}_{ig} + \widehat{\epsilon}_{u}$,
it suffices to have
\begin{align*}
  \wsqrt(\widehat{\epsilon}_{ig}, d^*)+ \wsqrt(\widehat{\epsilon}_{u}, d_a) +  \nsqrt(d^*) + \nsqrt(d_a)  
\leq \wsqrt(\hat{\epsilon}_{\mathrm{ERM}},d) + \nsqrt(d).
\end{align*}
This is equivalent to
\begin{align*}
  \sqrt{\widehat{\epsilon}_{ig}}\wsqrt(1, d^*)+ \sqrt{\widehat{\epsilon}_{u}}\wsqrt(1, d_a) +  \nsqrt(d^*) + \nsqrt(d_a) 
  \leq \sqrt{\hat{\epsilon}_{\mathrm{ERM}}}\wsqrt(1,d) + \nsqrt(d).
\end{align*}
Since $\widehat{\epsilon}_{ig} \leq \hat{\epsilon}_{\mathrm{ERM}}$,
it suffices to have
\begin{align*}
  \sqrt{ \hat{\epsilon}_{\mathrm{ERM}}}\wsqrt(1, d^*)+ \sqrt{\widehat{\epsilon}_{u}}\wsqrt(1, d_a) +  \nsqrt(d^*) + \nsqrt(d_a)  
  \leq \sqrt{\hat{\epsilon}_{\mathrm{ERM}}}\wsqrt(1,d) + \nsqrt(d),
\end{align*}
which is equivalent to
\begin{align*}
  \sqrt{\widehat{\epsilon}_{u}} \leq \sqrt{\hat{\epsilon}_{\mathrm{ERM}}}\cdot\frac{\wsqrt(1,d)-\wsqrt(1, d^*)}{\wsqrt(1, d_a)}+ \frac{\nsqrt(d)- \nsqrt(d^*) - \nsqrt(d_a)}{\wsqrt(1, d_a)}.
\end{align*}
For a fixed $\delta$ and large $d,d*,d_a$, this is equivalent to
\begin{align*}
  &\sqrt{\widehat{\epsilon}_{u}} \leq \sqrt{\hat{\epsilon}_{\mathrm{ERM}}} \cdot \Theta(\frac{\sqrt{d}-\sqrt{d^*}}{\sqrt{d_a}})- \sqrt{\frac{\log(m)}{m}}\cdot \Theta(\frac{d_a + d^* - d}{\sqrt{d_a}}),
\end{align*}
as claimed.

\subsection{Proof Of The Necessary Condition}
\label{app:secondalpha}
We now prove \thmref{secondalpha}. First, we prove an additional lemma that provides a necessary condition for $\ourbound \leq \ermbound$.

\begin{lemma}\label{lem:boundsfirst}
For any fixed $\delta \in (0,1)$, if $\ourbound \leq \ermbound$ then
  \[
  \sqrt{\widehat{\epsilon_{u} }} \leq \sqrt{\hat{\epsilon}_{\mathrm{ERM}}}\cdot\frac{\sqrt{d}}{\sqrt{d_a}}-\sqrt{\widehat{\epsilon}_{ig}}\cdot\frac{\sqrt{d^*}}{\sqrt{d_a}}  + o(1),
\]
 The convergence of the last term is with respect to the growth of $d,d^*$ together and/or of $m$.
\end{lemma}

\begin{proof}
By the assumption of the lemma, $\ourbound \leq \ermbound$. Thus, by definition,
\begin{align*}
  &\widehat{\epsilon}_{ig} + \widehat{\epsilon}_{u} + \wsqrt(\widehat{\epsilon}_{ig}, d^*) + \wsqrt(\widehat{\epsilon}_{u}, d_a)+  \nsqrt(d^*)+ \nsqrt(d_a)\leq  \hat{\epsilon}_{\mathrm{ERM}} + \sqrt{\hat{\epsilon}_{\mathrm{ERM}}}\wsqrt(1,d) + \nsqrt(d). 
\end{align*}
This is equivalent to:
\begin{align*}
  &\widehat{\epsilon}_{ig} + \widehat{\epsilon}_{u} + \sqrt{\widehat{\epsilon}_{ig}}\wsqrt(1, d^*) + \sqrt{\widehat{\epsilon}_{u}}\wsqrt(1, d_a)+  \nsqrt(d^*)+ \nsqrt(d_a)\leq  \hat{\epsilon}_{\mathrm{ERM}} + \sqrt{\hat{\epsilon}_{\mathrm{ERM}}}\wsqrt(1,d) + \nsqrt(d). 
\end{align*}
Therefore,
\begin{align*}
  \sqrt{\widehat{\epsilon}_{u} }\wsqrt(1,d_a)&\leq \hat{\epsilon}_{\mathrm{ERM}} -(\widehat{\epsilon}_{ig} + \widehat{\epsilon}_{u}) + \sqrt{\hat{\epsilon}_{\mathrm{ERM}}}\wsqrt(1,d)-  \sqrt{\widehat{\epsilon}_{ig}}\wsqrt(1,d^*) + \nsqrt(d) - \nsqrt(d_a) - \nsqrt(d^*) \notag \\
  &\leq \sqrt{\hat{\epsilon}_{\mathrm{ERM}}}\wsqrt(1,d) -  \sqrt{\widehat{\epsilon}_{ig}}\wsqrt(1,d^*)  + \nsqrt(d) - \nsqrt(d^*) - \nsqrt(d_a). 
\end{align*}
The last inequality follows since by \lemref{difficulty},
$
\hat{\epsilon}_{\mathrm{ERM}} \leq \widehat{\epsilon}_{ig} + \widehat{\epsilon}_{u}
$.

Now, from the definition of $\nsqrt$, we have
\begin{align*}
  &\nsqrt(d) - \nsqrt(d^*) - \nsqrt(d_a) = 8\frac{\log(m+1)}{m}(d-d^*-d_a)-\frac{4\log(\frac{4}{\delta})}{m}.
\end{align*}
By \lemref{vcF}, $d_a \geq d+d^*-2$. Thus, $d-d^*-d_a < 0$. It follows that
$\nsqrt(d) - \nsqrt(d^*) - \nsqrt(d_a) < 0$. Therefore,
\[
  \sqrt{\widehat{\epsilon}_{u} }\wsqrt(1,d_a) \leq \sqrt{\hat{\epsilon}_{\mathrm{ERM}}}\wsqrt(1,d) -  \sqrt{\widehat{\epsilon}_{ig}}\wsqrt(1,d^*).
  \]

It follows that
\begin{align*}
  \sqrt{\widehat{\epsilon}_{u} } &\leq \frac{\sqrt{\hat{\epsilon}_{\mathrm{ERM}}}\wsqrt(1,d) -  \sqrt{\widehat{\epsilon}_{ig}}\wsqrt(1,d^*) } {\wsqrt(1,d_a) }
                                 = \frac{\sqrt{\hat{\epsilon}_{\mathrm{ERM}} (d + A)} - \sqrt{\widehat{\epsilon}_{ig} (d^* + A)}}{\sqrt{d_a + A}},
\end{align*}
Where $A := \log(4/\delta)/(2\log(m+1))$. Thus, 
\[
  \sqrt{\widehat{\epsilon_{u} }} \leq \sqrt{\hat{\epsilon}_{\mathrm{ERM}}}\cdot\frac{\sqrt{d}}{\sqrt{d_a}}-\sqrt{\widehat{\epsilon}_{ig}}\cdot\frac{\sqrt{d^*}}{\sqrt{d_a}}  + o(1),
  \]
  where convergence of the last term is with respect to the growth of $d,d^*$ together and/or of $m$.
\end{proof}

Next, we prove the theorem using the two lemmas above. 
\begin{proof}[Proof of \thmref{secondalpha}]
Assume that $\ourbound \leq \ermbound$. By \lemref{boundsfirst},
\[
  \sqrt{\widehat{\epsilon}_{u}}  \leq \sqrt{\hat{\epsilon}_{\mathrm{ERM}}}\cdot\frac{\sqrt{d}}{\sqrt{d_a}}-\sqrt{\widehat{\epsilon}_{ig}}\cdot\frac{\sqrt{d^*}}{\sqrt{d_a}} + o(1).
  \]
Denote $\alpha:=d^*/d$. Suppose that $\alpha \geq 1$ (otherwise the statement in the theorem clearly holds). We have
\[
\sqrt{\widehat{\epsilon}_{u}}  \leq \frac{\sqrt{d}}{\sqrt{d_a}} \cdot (\sqrt{\hat{\epsilon}_{\mathrm{ERM}}}- \sqrt{\alpha}\cdot \sqrt{\widehat{\epsilon}_{ig}} )+ o(1).
\]
Here, the convergence is under a fixed $\alpha$ with growing $d,d^*$ or $m$.
Since $d_a\geq d^* + d-2 = (1 + \alpha ) \cdot d-2$, we have
\begin{equation}\label{eq:mm}
\sqrt{\widehat{\epsilon}_{u}}  \leq \frac{1}{\sqrt{1+ \alpha}} \cdot (\sqrt{\hat{\epsilon}_{\mathrm{ERM}}}- \sqrt{\alpha}\cdot \sqrt{\widehat{\epsilon}_{ig}} )+o(1).
\end{equation}

Since $\sqrt{\widehat{\epsilon}_{u}}\geq0$,
we have
\[
  \sqrt{\widehat{\epsilon}_{\mathrm{ERM}}}- \sqrt{\alpha}\cdot \sqrt{\widehat{\epsilon}_{ig}}+o(1)\geq0.
  \]
Thus,
$\widehat{\epsilon}_{ig} \leq \hat{\epsilon}_{\mathrm{ERM}}/\alpha + o(1).$
Combining with \lemref{difficulty}, we get
\[\hat{\epsilon}_{\mathrm{ERM}} \leq \widehat{\epsilon}_{ig} + \widehat{\epsilon}_{u} \leq  \hat{\epsilon}_{\mathrm{ERM}}/\alpha + \widehat{\epsilon}_{u} + o(1).
\] 
Combining this with \eqref{mm}, it follows that
\begin{align*}
  &\sqrt{(1-\frac{1}{\alpha})\cdot\hat{\epsilon}_{\mathrm{ERM}}} \leq \sqrt{\widehat{\epsilon}_{u}}\leq \frac{1}{\sqrt{1+ \alpha}} \cdot (\sqrt{\hat{\epsilon}_{\mathrm{ERM}}}- \sqrt{\alpha}\cdot \sqrt{\widehat{\epsilon}_{ig}}) + o(1).
\end{align*}
Rearranging, we get
\[
\sqrt{1+\alpha} \cdot (1-\frac{1}{\alpha})\leq 1 - \sqrt{\alpha} \cdot \frac{\sqrt{\widehat{\epsilon}_{ig}}}{\sqrt{\hat{\epsilon}_{\mathrm{ERM}}}} + o(1).
\]
This leads to

\[
\sqrt{\alpha}\frac{\sqrt{\widehat{\epsilon}_{ig}}}{\sqrt{\hat{\epsilon}_{\mathrm{ERM}}}} \leq 1-\sqrt{1+\alpha} \cdot (1-\frac{1}{\alpha}) + o(1).
\]
Since $0 \leq \frac{\sqrt{\widehat{\epsilon}_{ig}}}{\sqrt{\hat{\epsilon}_{\mathrm{ERM}}}}$, it must hold that $\sqrt{1+\alpha} \cdot (1-\frac{1}{\alpha}) \leq 1 + o(1)$. Solving for $\alpha$, we obtain that $\alpha \leq 2.25 + o(1)$.

Since $\alpha=d^*/d$, we conclude that  $d^* \leq 2.25 \cdot d + o(1)$, as claimed.
\end{proof}

\end{document}